\documentclass{article}

\usepackage{arxiv}

\usepackage[utf8]{inputenc} 
\usepackage[T1]{fontenc}    
\usepackage{hyperref}       
\usepackage{url}            
\usepackage{booktabs}       
\usepackage{amsfonts}       
\usepackage{nicefrac}       
\usepackage{microtype}      
\usepackage{lipsum}		
\usepackage{graphicx}
\usepackage{natbib}
\usepackage{doi}

\usepackage{caption} 
\usepackage{subcaption}
\usepackage{multirow}
\usepackage{isomath}

\usepackage{amsmath, amsthm, amssymb}
\newtheorem{prop}{Proposition}
\usepackage[font=normalsize,labelfont={bf}]{caption}
\usepackage{algpseudocode}
\usepackage{algorithm}

\title{Autoencoding a Soft Touch to Learn Grasping\\ from On-land to Underwater\thanks{A preprint submitted to the \textbf{Advanced Intelligent Systems} for review.}}

\author{
    {Ning Guo, Xudong Han, Xiaobo Liu, Jiansheng Dai}\\
    Department of Mechanical and Energy Engineering\\
    Southern University of Science and Technology\\
    Shenzhen, China 518055\\
    \And
    {Shuqiao Zhong, Zhiyuan Zhou, Jian Lin}\\
    Department of Ocean Science and Engineering\\
    Southern University of Science and Technology\\
    Shenzhen, China 518055\\
    \And
    {Fang Wan*}\\
    School of Design\\
    Southern University of Science and Technology\\
    Shenzhen, China 518055\\
    \texttt{wanf@sustech.edu.cn} \\
    \And
    {Chaoyang Song*}\\
    Department of Mechanical and Energy Engineering\\
    Southern University of Science and Technology\\
    Shenzhen, China 518055\\
    \texttt{songcy@ieee.org} \\
}

\renewcommand{\shorttitle}{Autoencoding a Soft Touch to Learn Grasping from On-land to Underwater}

\hypersetup{
pdftitle={Autoencoding a Soft Touch to Learn Grasping from On-land to Underwater},
pdfsubject={q-bio.NC, q-bio.QM},
pdfauthor={Ning Guo, Xudong Han, Xiaobo Liu, Shuqiao Zhong, Zhiyuan Zhou, Jian Lin, Jiansheng Dai, Fang Wan, Chaoyang Song},
pdfkeywords={Underwater Grasping, Soft Robotics, Tactile Learning},
}

\begin{document}
\maketitle

\begin{abstract}
    Robots play a critical role as the physical agent of human operators in exploring the ocean. However, it remains challenging to grasp objects reliably while fully submerging under a highly pressurized aquatic environment with little visible light, mainly due to the fluidic interference on the tactile mechanics between the finger and object surfaces. This study investigates the transferability of grasping knowledge from on-land to underwater via a vision-based soft robotic finger that learns 6D forces and torques (FT) using a Supervised Variational Autoencoder (SVAE). A high-framerate camera captures the whole-body deformations while a soft robotic finger interacts with physical objects on-land and underwater. Results show that the trained SVAE model learned a series of latent representations of the soft mechanics transferrable from land to water, presenting a superior adaptation to the changing environments against commercial FT sensors. Soft, delicate, and reactive grasping enabled by tactile intelligence enhances the gripper's underwater interaction with improved reliability and robustness at a much-reduced cost, paving the path for learning-based intelligent grasping to support fundamental scientific discoveries in environmental and ocean research.
\end{abstract}


\section{Introduction}
\label{sec:Intro}

    Collecting delicate deep-sea specimens of geological or biological interests with robotic grippers and tools is central to supporting fundamental research and scientific discoveries in environmental and ocean research \citep{Feng2022DeepSea, Gong2021Soft}. The human fingers are dexterous in object manipulation thanks to the finger's musculoskeletal biomechanics and skin's tactile perception even in harsh environments such as underwater \citep{Billard2019Trends, Kirthika2019Review}. Much research has been devoted to skilled object manipulation in daily life scenarios \citep{Ciocarlie2009Hand}. However, limited research focuses on transferring such capabilities to an underwater environment \citep{Mura2018Soft}. The ambient environment significantly challenges visual and tactile feedback integration while performing object grasping for visual identification under fluidic interference on the surface of physical interaction \citep{Capalbo2022Soft, Galloway2016Soft}. As a challenging task for humans, designing and developing robotic solutions for reactive and reliable grasping becomes even more complicated when the end-effector is fully submerged underwater \citep{Stuart2017Ocean}.

\subsection{Design towards Soft Grasping for Ocean Exploration}
\label{sec:Intro-UnderwaterGrasping}

    Object grasping is essential for environmental and ocean research to collect in situ specimens, where a trend towards softness in gripper design shows a growing adoption over the years \citep{Licht2016Universal}. Classical research on underwater grasping mainly focused on a direct translation of mechanical grippers made from rigid materials with waterproof design for all components, including the actuators, mechanisms, and sensors, resulting in a bulky design that is usually difficult for system integration \citep{Yuh2000Design}. Previous research reports a modular continuum finger for dexterous sub-sea manipulation with force and slip sensing, where the complex integration of a range of mechanical, electrical, and computing subsystems limits the use of this prototype out of a laboratory testing tank \citep{Lane1999Amadeus}. Another submarine gripper was developed as part of the European project TRIDENT \citep{Bemfica2014Three}, demonstrating dexterity for executing grasping and manipulation activities, but suffers from challenges when interacting with the delicate subsea environment and objects.
   
    Recent development in soft robotics adopts a different approach to leverage material softness for grasping \citep{Wang2023Inflatable}. The advantage of soft grippers for underwater scenarios is a systematic integration of fluidic actuation, motion transmission, and form-closed adaptation enabled by the soft, lightweight, low-cost material and fabrication against an aquatic environment with a reduced complexity using simple open-loop control \citep{Gong2021Soft}. These soft grippers had demonstrated successful, compliant interaction with various objects underwater \citep{Wang2021Safe}. A recent review shows an emerging research gap in introducing sensory capabilities to soft grippers underwater for closed-loop grasping feedback \citep{Mazzeo2022Marine}.
    
\subsection{The Need for Vision-based Tactile Grasping Underwater}
\label{sec:Intro-ReviewTactile}

    Inspired by the tactile perception of human fingers, a wide range of robotic research has been devoted to integrating with object grasping in industrial or daily life settings \citep{Bao2023Integrated}. Current research on tactile perception often leverages material softness for skin-like design \citep{Kirthika2019Review}. Recent work in 3D Tactile Tensegrity expanded the adoption of tactile sensing to the underwater environment, presenting a promising direction through the integration of soft self-powered triboelectric nanogenerators and deep-learning-assisted data analytics for underwater exploration \citep{Xu2023Deep}. While recent work gave an exhaustive investigation of tactile sensors and their applications in intelligent systems \citep{Liu2020Recent}, there is also an emerging field of vision-based tactile sensors in robotics under-represented in this field \citep{Li2022Implementing}. 

    Vision-based tactile perception leverage machine vision to provide multi-modal contact information with detailed spatial resolution \citep{Shimonomura2019Tactile}. The focus is to deploy soft mediums that deform under external forces and infer tactile information from visual observation \citep{Zhang2022Hardware}. Sato et al. \citep{Sato2010Finger} built a linear approximation model to estimate the contact forces by tracking two colored spherical markers arranged at different depths of an elastomer surface. Yamaguchi et al. \citep{Yamaguchi2019Tactile} presented another low-order approximation model to infer the contact forces from observed makers' variations in the camera. Unfortunately, the current literature has not yet explored the adoption of vision-based tactile robotics in the underwater scenario. 

\subsection{Machine Learning for Latent Intelligence in Tactile Robotics}

    The performance of machine learning algorithms heavily depends on the choice of data representation \citep{Bengio2014Representation}. When projecting complex soft robotics deformation into image space \citep{Yuan2017Gelsight}, a growing trend of research is devoted to treating the representation of a captured image as the latent variables of an appropriate generative model \citep{Doersch2016Unsupervised}. The generative models are usually highly interpretable in understanding the causal relations of the observations \citep{Kingma2014AutoEncoding}, making it a potential solution to increase the robustness of vision-based, soft, tactile sensing underwater where the environmental uncertainties are much worse than the daily life or industrial settings, as shown in \textbf{Figure \ref{fig:PaperOverview}}.

   \begin{figure}[!ht]
        \centering
        \includegraphics[width=1\textwidth]{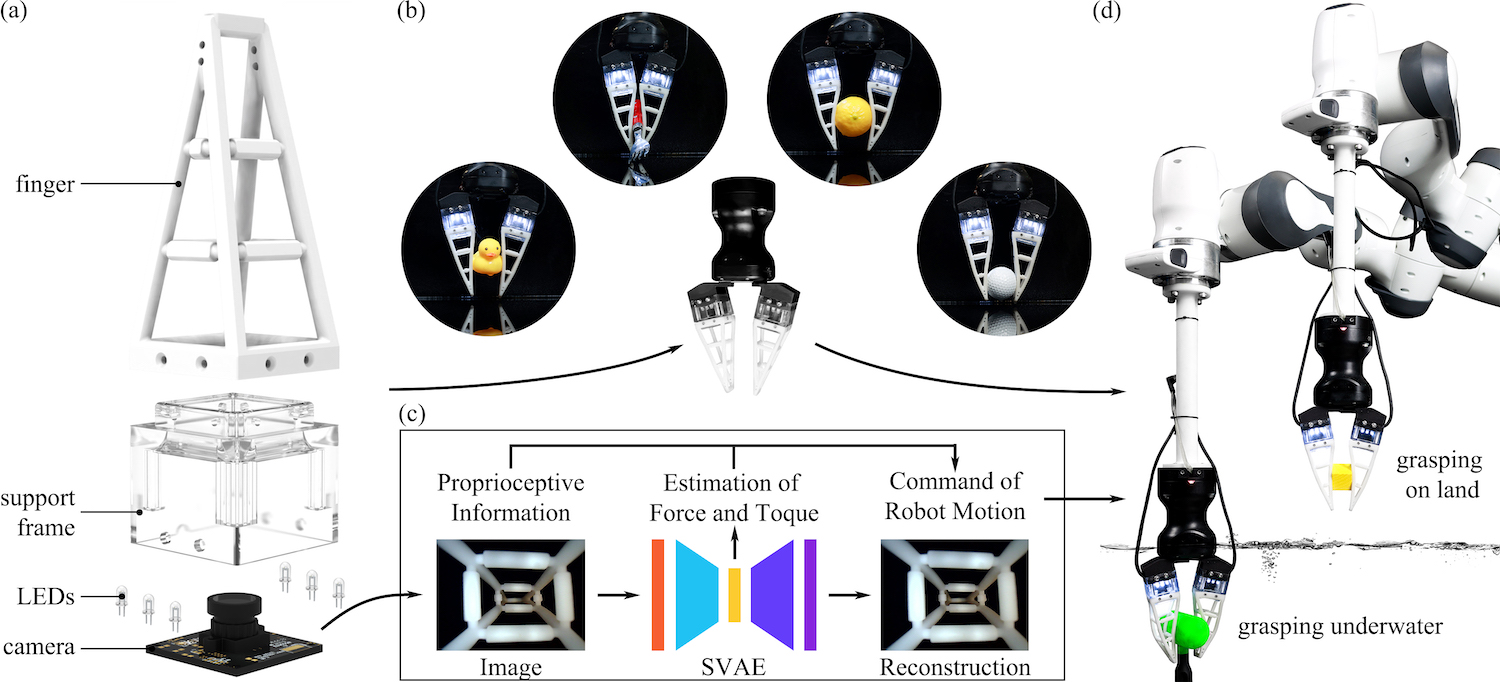}
        \caption{
            \textbf{Overview of the soft visual-tactile learning across land and water using SVAE.} 
            (a) Design of the sensorized soft finger where the camera board is sealed with a silicon layer. 
            (b) The integrated amphibious gripper is transformed by replacing the fingertip of a Robotiq Hand-E gripper with the sensorized soft finger with omni-directional grasping adaptations. The Hand-E gripper has an Ingress Protection (IP) rating of IP67, which is suitable for our underwater experiment in a tank without extra waterproofing. 
            (c) The scheme of visual-tactile learning takes an image of the deformed metamaterial as input, reconstructs the image, and simultaneously predicts the force and torque. 
            (d) The amphibious gripper is mounted on a Franka Emika Panda robot arm to execute force control tasks on land and underwater.
            }
        \label{fig:PaperOverview}
    \end{figure}
    
    Variational autoencoder (VAE) recently emerged as a powerful generative model that learns the distribution of latent variables and is widely used for visual representation in robot learning \citep{Rezende2014Stochastic, Kingma2014AutoEncoding, Takahashi2019Variational}. Since the original publication, many variants and extensions of VAE have been proposed. Semi-supervised VAE was proposed to address the problem of unlabeled data training \citep{Kingma2014SemiSupervised}. Higgins et al. \citep{Higgins2016betaVAE} introduced weight to balance the reconstruction error and regularization of latent variable distribution, enabling learning of a disentangled latent representation. The recent adoption of a supervised VAE model for identifying critical underlying factors for prediction demonstrates the promising potential for application in robotic grasping \citep{Ji2021Multi}.

    This study investigates the transferability of grasping knowledge from on-land to underwater via a vision-based soft robotic finger that learns 6D forces and torques (FT) using a Supervised Variational Autoencoder (SVAE). Using real-time images collected from an in-finger camera that captures the soft finger's whole-body deformations while interacting with physical objects on-land and underwater, we established a learning-based approach to introduce tactile intelligence for soft, delicate and reactive grasping underwater, making it a promising solution to support scientific discoveries in interdisciplinary research.

\section{Results}
\label{sec:Results}
\subsection{In-Finger Vision for a Soft Tactile Finger Underwater}
\label{sec:Results-Design}

    Here we present the in-finger vision design for tactile sensing compatible with both on-land and underwater scenarios, as shown in Figure \ref{fig:PaperOverview}(a). The finger is based on a soft metamaterial with a shrinking cross-sectional geometry towards the tip, capable of omni-directional adaptation on the finger surface to unknown object geometries, enabling a passive form-closure for robotic grasping \citep{Wan2022VisualLearning}. A monocular RGB camera (120 frames per second) is mounted inside a support frame under the finger to obtain high-framerate images of the finger's adaptive deformations at a resolution of $640\times480$ pixels. The support frame is 3D printed with the optically transparent material (Somos\textsuperscript{\textregistered} WaterShed XC 11122). All electronics inside are waterproofed by dipping the camera board, except the lens, into transparent silicon. We added six LEDs to the camera board for improved lighting conditions, resulting in an integrated design of a water-resistant, soft robotic finger with machine vision from the inside.

    Figure \ref{fig:PaperOverview}(b) shows the integration of the proposed finger with a Robotiq's Hand-E gripper, which has an ingress protection rating of IP67 for testing in lab tanks. The proposed soft finger exhibits spatial adaptive deformations, conforming to the object's geometry during physical contact and exhibiting both regular and twisted adaptions for enhanced robustness for grasping, as shown in Figure \ref{fig:PaperOverview}(d). For more intensive use in the field, one can directly mount the soft finger to the tip of existing grippers on an underwater vehicle. We demonstrated the effeteness of using the soft finger by grasping some Yale-CMU-Berkeley (YCB) objects of various shapes and softness underwater or floating on the water surface \citep{Calli2015Benchmarking}. See Movie S1 in the Supplementary Materials for further details. 

    The advantage of the proposed design is a complete separation of the sensory electronics from the soft interaction medium by design, resolving the issues of an enclosed chamber that may suffer from severe surface deformation when used underwater \citep{Li2022Implementing}. Such design enables us to collect real-time image streams of the physical interaction between the soft finger and external object using the in-finger vision, as shown in Figure \ref{fig:PaperOverview}(c), which can be further implemented with generative models, such as the supervised variational autoencoder (SVAE), to provide the tactile perception of grasping interactions, both on land and underwater.

\subsection{Generative Tactile Learning via Supervised Variational Autoencoder}
\label{sec:Results-Learning}

    Here presents a generative learning architecture for tactile perception in both on-land and underwater scenarios with latent explanations using a supervised variational autoencoder (SVAE) in \textbf{Figure \ref{fig:SvaeArchitecture}}.

    \begin{figure}[!ht]
        \centering
        \includegraphics[width=1\textwidth]{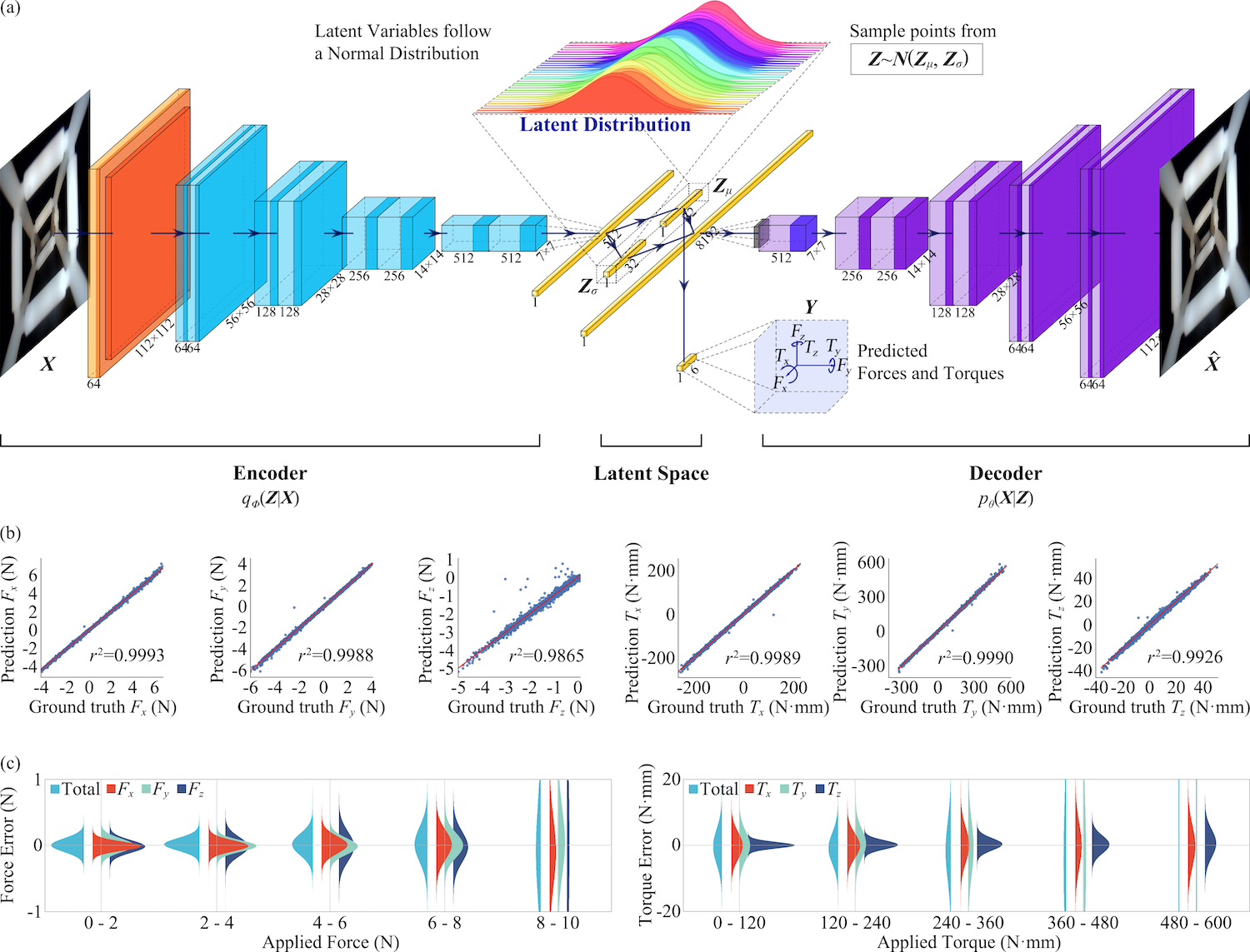}
        \caption{
            \textbf{Latent deformation learning model for the soft metamaterial.} 
            (a) The architecture of the Supervised Variational Autoencoder (SVAE) model, where a VAE is combined with a supervised regression task for force and torque prediction. 
            (b) Predicted force/torque versus the ground truth in each of the six dimensions on the test dataset. 
            (c) Distributions of prediction errors in each 6D force/torque dimension over different ranges.
            }
        \label{fig:SvaeArchitecture}
    \end{figure}

    The generative model is illustrated in Figure \ref{fig:SvaeArchitecture}(a), which includes an encoder $q_\phi(\vectorsym{Z}|\vectorsym{X})$ to process real-time images from the in-finger vision, then processed through a latent space operation to estimate a latent distribution of the interactive physics $\vectorsym{Z}\sim{\vectorsym{N}(\vectorsym{Z}_\mu|\vectorsym{X}_\theta)}$, assuming a normal distribution. Here, we added a force and torque prediction based on the $\vectorsym{Z}_\mu$ to produce the 6D tactile estimation as an auxiliary output. Finally, through a generative decoder $p_\theta(\vectorsym{Z}|\vectorsym{X})$, our model reproduces images of the tactile interactions based on the learned SVAE model. Note that $\theta$ and $\phi$ are the parameters of the encoder and decoder neural networks, which must be optimized during training. Note that the SVAE model's loss function for training is the combination of image reconstruction loss, force/torque prediction loss, and latent representation regularization loss. Following the detailed formulation of the SVAE model in the Materials and Methods section, for tuning with small datasets, we introduced two hyper-parameters, $\alpha$ and $\beta$, to modify the objective function into the following, where the parameter $\alpha \geq{0}$ is used to adjust the relative importance during optimization between the image reconstruction and force/torque prediction tasks.
    \begin{equation}\label{J_alpha,x,y}
        \begin{split}
           \Tilde{L}(\theta,\phi;\vectorsym{X},\vectorsym{Y})= &-\frac{\alpha}{1+\alpha}\|\vectorsym{X}- \vectorsym{\hat{X}}\| - \frac{1}{1+\alpha}\|\vectorsym{Y}- \vectorsym{\hat{Y}}\| + \beta D_{KL}[\vectorsym{N}(\vectorsym{Z}_\mu,\vectorsym{Z}_\sigma)||\vectorsym{N}(\vectorsym{0},\vectorsym{I})],
        \end{split}
    \end{equation}

    Figure \ref{fig:SvaeArchitecture}(b) shows the predicted 6D forces and torques via SVAE against the ground truth. The $R^2$ scores are higher than 0.98 for 6D force and torque predictions, indicating the SVAE model's excellent performance in tactile sensing on the test dataset. We also plot the distributions of prediction errors in each 6D force/torque dimension over different ranges in Figure \ref{fig:SvaeArchitecture}(c). For applied forces ranging between $[0,2)$, $[2,4)$, $[4,6)$, $[6,8)$, and $[8,10)$ N, the standard deviations of the prediction are 0.07, 0.06, 0.09, 0.12, and 0.24 N, respectively. For applied torques ranging between $[0,120)$, $[120,240)$, $[240,360)$, $[360,480)$, and $[480,600)$ N$\cdot$mm, the standard deviations of the prediction are 4.6, 3.9, 6.0, 9.1, and 20.3 N$\cdot$mm, respectively. These results follow a Gaussian distribution with a near-zero mean and an increasing standard deviation as the range becomes more considerable. The force sensing errors are comparable in the $x$ and $y$ axes and more prominent in the $z$ axis, while the torque sensing errors are the least in the $z$ axis. This characteristic is primarily due to the metamaterial's structural design, which is less sensitive to the force along the $z$ axis. 

    We also conducted a comparative study to evaluate the proposed SVAE model against two baseline models, including a ConvNet model for force and torque prediction only and a VAE model for image reconstruction only, with results summarized in \textbf{Table \ref{tab:CompareSVAE}}. 
    \begin{table}[!ht]
        \caption{Comparative analysis of the proposed SVAE's performances.}
        \label{tab:CompareSVAE}
        \begin{center}
            \begin{tabular}{cccc}
            \hline
            \textbf{Models} &
              \textbf{Settings} &
              \textbf{\begin{tabular}[c]{@{}c@{}}Image Reconstruction \\ Error (MSE) \end{tabular}} &
              \textbf{\begin{tabular}[c]{@{}c@{}}Force/Torque Prediction \\ Accuracy (avg.)\end{tabular}} \\ \hline
            ConvNetV              & Vanilla          & $-$                 & $96.04\%$ \\ \hline
            \multirow{6}{*}{SVAE} & $\alpha = 0.001$ & $5.19\times10^{-2}$ & $99.53\%$ \\
                                  & $\alpha = 0.01$  & $5.17\times10^{-2}$ & $99.52\%$ \\
                                  & $\alpha = 0.1$   & $2.02\times10^{-2}$ & $99.53\%$ \\
                                  & $\alpha = 1$     & $9.47\times10^{-3}$ & $99.45\%$ \\
                                  & $\alpha = 10$    & $6.68\times10^{-3}$ & $96.22\%$ \\
                                  & $\alpha = 100$   & $5.65\times10^{-3}$ & $61.46\%$ \\ \hline
            VAE                   & Vanilla          & $5.36\times10^{-3}$ & $-$       \\ \hline
            \end{tabular}
        \end{center}
    \end{table}
    These two models share the same network architecture as the SVAE but are trained separately. However, the ConvNet model is a deep regression network with a convolutional layer that only takes force/torque prediction loss. And the VAE model is a Variational Autoencoder that only takes image reconstruction and latent representation regularization loss. We used mean square error (MSE) between the original and reconstructed images to evaluate the representation learning task and the coefficient of determination $R^2$ to assess the overall force/torque prediction task. See Methods S1 in the Supplementary Materials for further details on training data collection. 

    The SVAE has shown comparable performance over the vanilla VAE in the representation learning task while $\alpha$ is approaching infinity. Meanwhile, SVAE outperforms the deep regression model ConvNet in the force/torque prediction task when $\alpha\leq1$ and the training is focused more on the prediction task, achieving over $99.45\%$ on the validation set. Since SVAE is a multi-task learning framework, the hyper-parameter $\alpha$ is vital in balancing the reconstruction and prediction tasks. Here, $\alpha = 1$ is chosen for all validation tests and real-time experiments. The results show that the co-trained representation learning enhances the force/torque prediction task. See Movie S2 in the Supplementary Materials for a video demonstration of real-time 6D force/torque prediction using the SVAE model in on-land and underwater scenarios. 

\subsection{Land2Water Generalization of Tactile Representation}
\label{sec:Results-LatentRep}

    We also investigated the generalization of tactile representations learned via SVAE in a Land2Water skill transfer problem for tactile sensing in \textbf{Figure \ref{fig:RepresentationResults}}. 
    \begin{figure}[!ht]
        \centering
        \includegraphics[width=1\textwidth]{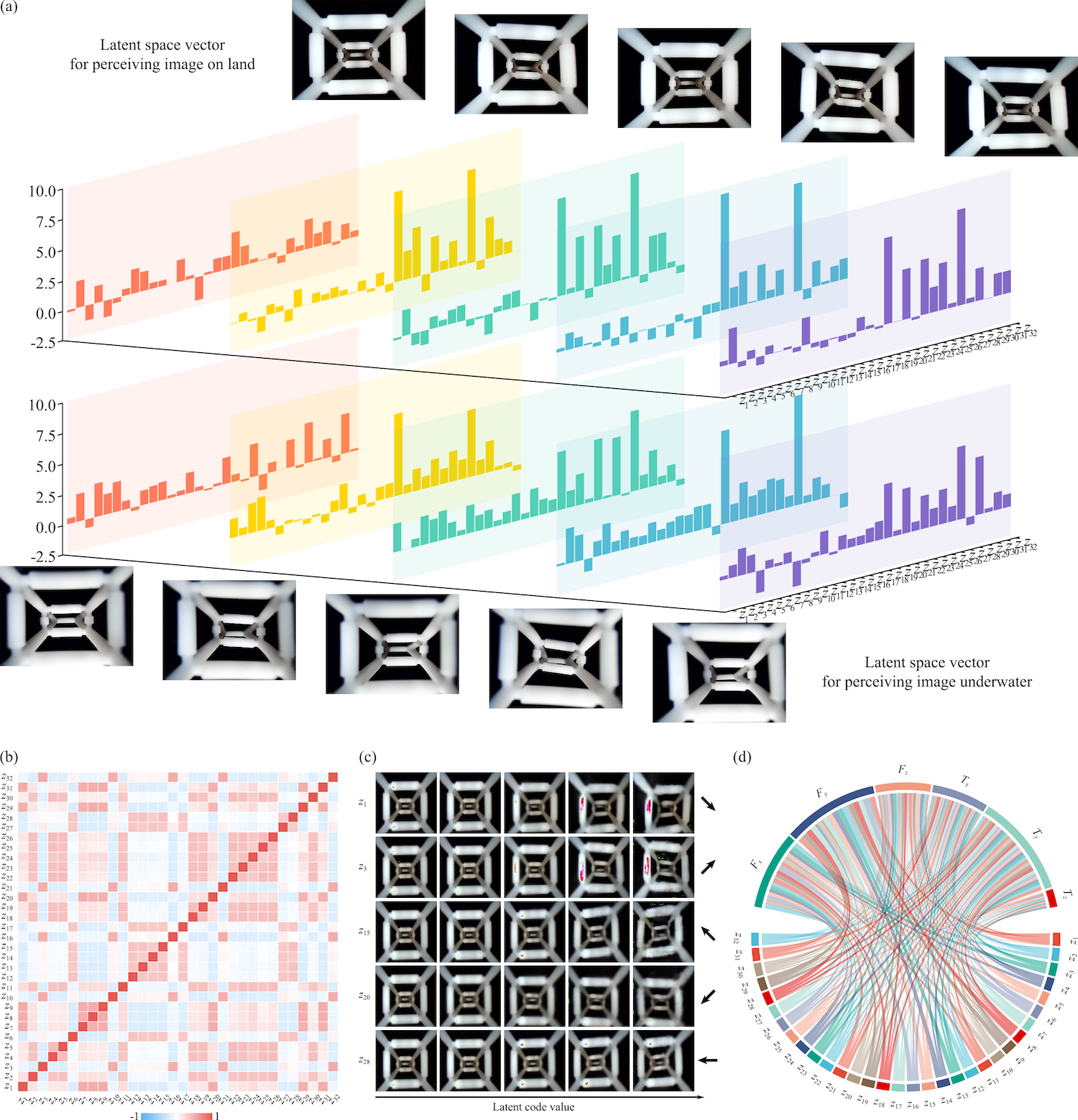}
        \caption{
            \textbf{Representation learning of deformations of the proposed soft metamaterial.} 
            (a) The complex deformations of the soft metamaterial on land and underwater are represented in the latent space.
            (b) Correlation map of learned 32 latent variables. 
            (c) Reconstructed images of varying selected latent variables. 
            (d) The relative correspondences between latent variables and force/torque.
            }
        \label{fig:RepresentationResults}
    \end{figure}
    While implementing the SVAE model, we chose a 32-dimension definition with a balanced trade-off between reconstruction error and dimensional complexity in explanatory power. See Methods S2 in the Supplementary Materials for further discussion. 

    Figure \ref{fig:RepresentationResults}(a) shows the comparison of the 32D latent space vectors for tactile perception between on-land (top) and underwater (bottom) scenarios when the soft finger is experiencing the same deformation delivered by the robotic arm. Five random instances of the in-finger vision are chosen for each scenario and plotted with their corresponding latent variable distributions. We identified a similar distribution between the upper and lower plots for these five random instances. This suggests the transferability of the latent variables' explanatory power in tactile perception between the on-land and underwater scenarios. This is because the learned latent representation could be close to the intrinsic dimension of the soft finger deformation, minimizing the information loss during tactile image encoding. The segment of the soft finger interacting with objects is made from 3D-printed metamaterial without any electronic parts, whose mechanical properties are not affected by the water, indicating the generalization of tactile representation in Land2Water transfer, which is reported for the first time in the vision-based tactile sensing literature.

    The correlation map plotted in Figure \ref{fig:RepresentationResults}(b) suggests that these 32 latent variables learned from our SVAE model are generally unrelated, which is a preferred property for representation learning \citep{Bengio2014Representation}. However, for variable clusters such as $\{Z_7, Z_8, Z_9\}$ and $\{Z_{12},..., Z_{15}\}$, regional correlation is observable at a relatively small scale. We also demonstrate the latent interpolation for the metamaterial's deformation projected in the image plane on selected dimensions of $\{Z_1$, $Z_3$, $Z_{15}$, $Z_{20}$, $Z_{28}\}$ in Figure \ref{fig:RepresentationResults}(c), which gives an intuitive sense of what are the latent variables represent in physical space. For example, we found that $Z_1$ and $Z_3$ are related to pushing right-downward and right-upward when their values go from negative to positive, while $Z_{15}$ and $Z_{20}$ are related to moving left-upward and left-downward. Furthermore, $Z_{28}$ has a prominent horizontal movement. These latent variables are strongly related to representing the complex deformations of the soft metamaterial in terms of image reconstruction but are not disentangled. As shown in Figure \ref{fig:RepresentationResults}(d), the correlation between the 6D force/torque and the 32D latent variables is complex and diversified. For example, the latent variable $Z_{28}$ strongly correlates with $F_y$, which agrees with the reconstructed horizontal movement along the corresponding axis of the camera coordinate. 

\subsection{Land2Water Grasping Knowledge Transfer}
\label{sec:Results-Realtime}
    
    This section presents two experiment results that implement the Land2Water grasping knowledge obtained through tactile sensing, including one for object grasping against location uncertainties and another for tactile sensorimotor grasping adaptability from on-land to underwater scenarios. 

\subsubsection{Object Grasping against Location Uncertainties} 
\label{sec:Results-ObjectGrasping}
    
    This experiment demonstrates the equal necessity of tactile perception when grasping underwater in \textbf{Figure \ref{fig:GraspingTest}}, which is generally acknowledged to increase the robustness in on-land conditions. 
    \begin{figure}[!ht]
        \centering
        \includegraphics[width=1\textwidth]{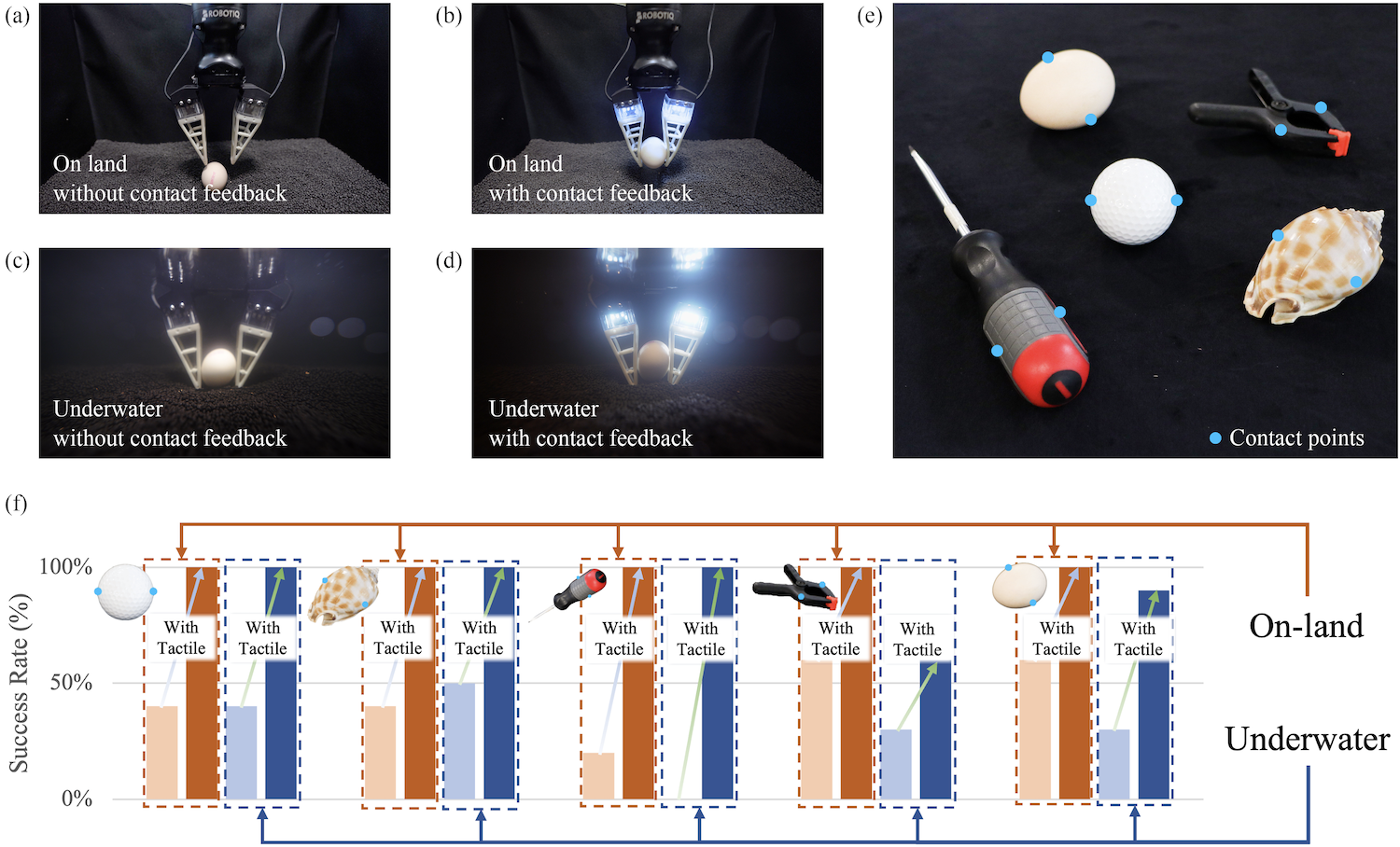}
        \caption{
            \textbf{Tactile grasping results with or without SVAE in on-land and underwater scenarios.} 
            (a) Open-loop object grasping on land with a predefined grasping position. 
            (b) Closed-loop grasping on land with contact force feedback. 
            (c) Open-loop object grasping underwater with a predefined grasping position. 
            (d) Closed-loop object grasping underwater with contact force feedback. 
            (e) Test objects with the predefined grasping points marked. 
            (f) The grasp result summary.
            }
        \label{fig:GraspingTest}
    \end{figure}
    Figures \ref{fig:GraspingTest}(a)\&(c) show the open-loop grasping without force feedback, where the gripper reaches the target grasping point, closes the fingers to a given gripping width and lifts the object. While in the case of closed-loop grasping with force feedback, as illustrated in Figures \ref{fig:GraspingTest}(b)\&(d), the gripper adjusts the gripping width according to the force estimation from SVAE until a grasping confirmation signal is triggered, then lifts the object. While modern learning-based methods are effective in performing grasp planning \citep{Morrison2020Learning, Mahler2019Learning}, here, we manually selected the gripper's grasping positions and gripping widths for each test object shown in Figure \ref{fig:GraspingTest}(e) for the ease of comparison. Figure \ref{fig:GraspingTest}(f) summarizes the ten grasping trials for each object using both methods and reports the average success rate. We added a standard deviation $\sigma = 5$ mm to simulate the uncertainty due to grasping parameters prediction. The average success rate for on-land grasping of the five test objects is $44\%$ without contact feedback. After adding tactile feedback, the success rate is significantly enhanced to $100\%$. After adding tactile feedback, our results show a similar enhancement for underwater grasping, boosting the average success rate from $30\%$ in open-loop grasping to $90\%$ in closed-loop grasping. See Movie S3 in the Supplementary Materials for a video demonstration.

\subsubsection{Tactile Sensorimotor Grasping Adaptability}     
\label{sec:Results-ForceAdjusting}

    This experiment demonstrates sensorimotor grasping using tactile perception enabled by the proposed SVAE model, which is transferable from on-land to underwater scenarios. \textbf{Figure \ref{fig:Control}}(a) demonstrates the overall experiment process. 
    \begin{figure}[!ht]
        \centering
        \includegraphics[width=1\textwidth]{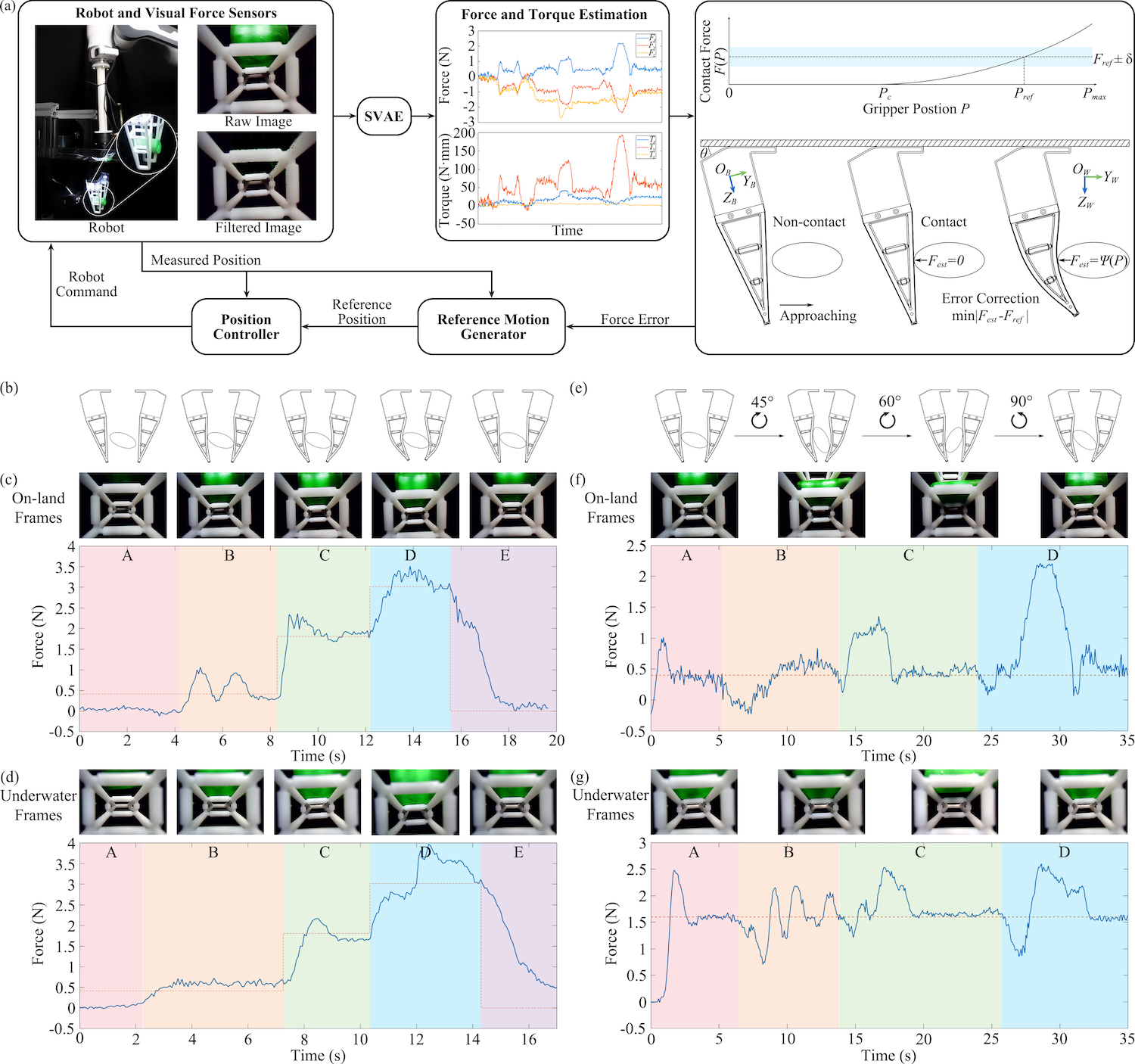}
        \caption{
            \textbf{Tactile perception of soft finger for real-time robotic grasping control.} 
            (a) Experiment setup of the force control tasks with soft tactile sensing. The goal of the force control tasks is to maintain the contact force at the required values by controlling the position of the soft finger. 
            (b)-(d) Desired gripping force following experiments. The gripping force is commanded to a series of expected values, and the corresponding gripper adaptation stages are illustrated in different colors. 
            (e)-(g) In-hand object shape adaptation experiments. The grasped object's shape changes constantly, and the gripper sensitively adjusts its position to maintain the constant gripping force.
            }
        \label{fig:Control}
    \end{figure}
    Once contact begins, the in-finger vision captures the whole-body deformations of the soft finger and feeds the SVAE model with real-time image streams of the physical interaction at 120 Hz. The 6D forces and torques are predicted for both on-land and underwater scenarios, then compared against a pre-defined threshold for reactive grasping. During this process, the width between the two soft fingers is actively adjusted to accommodate the disturbances in object status, i.e., fluidic disturbances for grasping underwater and sudden collision for on-land grasping. We execute the reactive grasping by sending reference position commands to a position controller in the robot system using a motion generator calculated by the measured gripper position and force error detected on the fingers.

    Figure \ref{fig:Control}(b) illustrates the experiment process that tests the gripper system's responsiveness of tactile-reactive grasping, a desirable capability for both on-land and underwater grasping of objects with known properties. After making contact with a slightly rotated tube, fixed, of oval cross-section, we send force commands to the gripper to maintain a contact force at 0.4, 1.6, and 3 N sequentially. Shown in Figures \ref{fig:Control}(c)\&(d) are the recorded force (in blue) against the commanded force (in red) when the experiment was conducted on land and underwater. In both scenarios, the force controller successfully transited and stabilized at the commanded contact force within seconds. See Movie S4 in the Supplementary Materials for a video demonstration. 

    Figure \ref{fig:Control}(e) illustrates another experiment that tests the gripper's capability to maintain a specified contact force while reacting to disturbances, a preferred but more challenging skill for both on-land and underwater grasping of objects with unknown yet delicate properties. In this experiment, the oval-shaped tube is commanded to rotate clockwise in 45$^{\circ}$ and 60$^{\circ}$ first and then counterclockwise in 90$^{\circ}$ to simulate the changing interaction between the gripper and target object. During the process, the gripper needs to maintain a 0.4 N force for the on-land experiment in Figure \ref{fig:Control}(f) and a 1.6 N force for the underwater experiment in Figure \ref{fig:Control}(g). When the target object changes its pose during rotation, the gripper reacts to the shape variation based on the estimated force from SVAE. See Movie S5 in the Supplementary Materials for a video demonstration. We also tested the gripper's reactive grasping under rotational disturbances by turning a cylinder along the $z$-axis. In this case, the SVAE model successfully predicted a torque while the fingers started twisting and commanded the gripper to rotate while maintaining a zero torque $\tau_z$ in reactive motion. See Movie S6 in the Supplementary Materials for a video demonstration.

\section{Discussion}
\label{sec:Discussion}

    It has been a challenge to introduce robotic intelligence into underwater grasping by adding the sense of touch \citep{Mazzeo2022Marine}, which supports delicate and autonomous interactions with the unstructured aquatic environment for scientific activities in environmental, biological, and ocean research. Classical solutions usually take a mechanical approach with various sealing technologies to deal with fluidic pressure and corrosive contamination, suffering trade-offs in engineering flexibility and intelligent perception. This work proposes a vision-based approach to achieve high-performing tactile sensing underwater by combining the emerging advancement in soft robotics and machine learning. The simplicity of the design enables a minimum set of mechanical components, avoiding dynamic seals for enhanced robustness underwater. The soft finger's passive adaptation and in-finger vision enable a seamless integration of the proposed Supervised Variational Autoencoder (SVAE) to learn tactile sensing through visual sensing underwater. The latent representations learned from the SVAE algorithm enable a generative solution to infer the 6D forces and torques during physical interactions underwater with explanatory reasoning. As a result, we successfully transferred the tactile intelligence of the proposed gripper system from on-land to underwater. We achieved tactile force prediction accuracy above $98\%$ along each axis on the test set, using the same hardware with minimal algorithmic parameter adjustment. Real-time grasping experiment results in a lab tank demonstrate the effectiveness of the soft tactile finger for reliable and delicate grasping in both environments. 

    Model explainability and generalization are primarily concerned in machine learning research. Considering the transferability of tactile intelligence from on-land to underwater, we leverage the variational autoencoder model's powerful representation learning capability to express the soft finger's deformation patterns in latent space. Experiment results show that the extracted latent features of the same finger deformation in different environments exhibit a similar distribution. From the statistical inference perspective, learning this low-dimensional deformation pattern is closely related to the dimension reduction problem \citep{Adragni2019Sufficient} where the learned latent representation corresponds to an approximated sufficient statistics of original data \citep{Joyce2008Approximately}. In contrast with conventional dimension reduction techniques such as Principal Component Analysis, the convolutional neural network usually performs better in finding the low-dimensional representation from image data \citep{Hinton2006Reducing}.
    
    Performance degradation of tactile force prediction from on-land to underwater is unavoidable due to the significant change of visual input to the SVAE model. Due to unpredictable fluid dynamics, object grasping underwater is generally more challenging than on land, which is the same case with or without tactile feedback, as demonstrated in our experimental results. However, adding tactile feedback to the gripper system effectively enhanced the reliability of underwater grasping. The finger's network design cuts the fluids while closing, generally causing fewer disturbances for underwater grasping, a common problem usually suffered by fingers with a rigid structure \citep{Stuart2017Ocean}. 

    Tactile perception is generally desired to achieve effective grasping behaviors in underwater environments but is under-explored in research and practice compared with the on-land scenarios \citep{Mazzeo2022Marine, Yan2021Soft}. Operational tasks for underwater robotics are usually associated with a lack of vision, leading to ambiguous recognition of objects \citep{Subad2021Soft}, in which tactile perception plays an important role. Our results in grasping success rates demonstrate the benefit of tactile perception when visual perception is underperforming. Besides, reactive control architecture based on the perception-action cycle can be integrated with our tactile soft finger to achieve more intelligent manipulation underwater.

    Our presented work has several limitations, which need future research for optimization in structural design and learning algorithms. For example, visual input tends to be corrupted by background noise in an underwater environment, which could be alleviated mechanically by adding a layer of silicone skin on the finger surface \citep{Jiang2021RigidSoft}. We could also enhance the tactile perception using XMem \citep{Cheng2022XMem} to track the soft finger's deformation from the in-finger vision or use inpainting algorithms \citep{Yu2023InpaintAnything} to use the in-finger vision for visual perception. The proposed underwater grasping system is yet to be tested on Remotely Operated Vehicles (ROVs) in shallow and deep water for further engineering enhancement.

\section{Materials and Methods}
\label{sec:Method}
\subsection{Formulating the Supervised Variational Autoencoder}
\label{sec:SVAE-Design}

    Accurately deriving the relationship between deformation and force of soft structure can significantly improve the efficacy of visual-tactile sensing \citep{Ma2019Dense}. However, the geometry-dependent deformation of the soft structure is complex to represent. Even though we can discretize the structure with standard node elements using the Finite Element Method, measuring the displacements of corresponding nodes from a monocular camera can be another problem.

    The standard solution involves a two-step method by first building a force-displacement mapping of soft structure and then solving the partial observable vision problem using a monocular camera \citep{Dong2022Towards}. Here, we leverage the interpretability of latent variables in the original VAE model and constrain these learned factors to image-based features of our soft finger deformation using in-finger vision, where the restored force can be measured during training and act as a supervised signal to guide the learning of latent space.

    As shown in Figure \ref{fig:SvaeArchitecture}(a), suppose the collected, labeled data pairs $(\vectorsym{X}, \vectorsym{Y})$ are independent and identically distributed, where $\vectorsym{X}$ and $\vectorsym{Y}$ are images and vectors of force/torque, respectively. The aim is to find an optimal representation $\vectorsym{Z}$ of $\vectorsym{X}$ containing sufficient information about $\vectorsym{Y}$. To tackle both representation learning and force/torque prediction tasks, we extend the optimization framework of the original VAE \citep{Kingma2014AutoEncoding} to an additional supervised task and maximize the log-likelihood function of marginal probability $\log{p_\theta(\vectorsym{X}, \vectorsym{Y})}$:
    \begin{equation}
            \log{p_\theta(\vectorsym{X},\vectorsym{Y})} = L(\theta,\phi;\vectorsym{X},\vectorsym{Y}) + D_{KL}[q_\phi(\vectorsym{Z}|\vectorsym{X})||p_\theta(\vectorsym{Z}|\vectorsym{X},\vectorsym{Y})],
    \end{equation}
    where $L(\theta,\phi;\vectorsym{X},\vectorsym{Y})$ is the evidence lower bound (ELBO) for SVAE, which can be extended as:
    \begin{equation}\label{elbo_x,y}
        \begin{split}
            \log{p_\theta(\vectorsym{X},\vectorsym{Y})}& \geq L(\theta,\phi;\vectorsym{X},\vectorsym{Y})\\ &=E_{\vectorsym{Z}\sim{q_\phi(\vectorsym{Z}|\vectorsym{X})}}[\log{p_\theta(\vectorsym{X}|\vectorsym{Z})}]+E_{\vectorsym{Z}\sim{q_\phi(\vectorsym{Z}|\vectorsym{X})}}[\log{p_\theta(\vectorsym{Y}|\vectorsym{Z})}]-D_{KL}[q_\phi(\vectorsym{Z}|\vectorsym{X})||p_\theta(\vectorsym{Z})].
        \end{split}
    \end{equation}

    In Equation \ref{elbo_x,y}, for continuous data of image, force, and latent variables $\vectorsym{X}, \vectorsym{Y}, \vectorsym{Z}$, the prior distribution of the latent variables $p_\theta(\vectorsym{Z})$, distribution of probabilistic encoder $q_\phi(\vectorsym{Z}|\vectorsym{X})$ and decoder $p_\theta(\vectorsym{X}|\vectorsym{Z})$, $p_\theta(\vectorsym{Y}|\vectorsym{Z})$ are assumed to follow a normal distribution:
    \begin{equation}
        \begin{split}
             p_\theta(\vectorsym{Z})&\sim{\vectorsym{N}(\vectorsym{0},\vectorsym{I})}, \\ 
             q_\phi(\vectorsym{Z}|\vectorsym{X})&\sim{\vectorsym{N}(\vectorsym{Z}_\mu(\vectorsym{X},\phi),\vectorsym{Z}_\sigma(\vectorsym{X},\phi))},\\
             p_\theta(\vectorsym{X}|\vectorsym{Z})&\sim{\vectorsym{N}(\vectorsym{X}_\mu(\vectorsym{Z},\theta),\vectorsym{I})},\\
             p_\theta(\vectorsym{Y}|\vectorsym{Z})&\sim{\vectorsym{N}(\vectorsym{Y}_\mu(\vectorsym{Z},\theta),\vectorsym{I})}.
        \end{split}
    \end{equation}
    Maximization of the new ELBO in Equation \ref{elbo_x,y} is equivalent to maximizing the following optimization object, where the outputs from two decoders are denoted as $\vectorsym{\hat{X}}$ and $\vectorsym{\hat{Y}}$, respectively.
    \begin{equation}\label{J_x,y}
        \begin{split}
           \Tilde{L}(\theta,\phi;\vectorsym{X},\vectorsym{Y})= &-\|\vectorsym{X}- \vectorsym{\hat{X}}\| - \|\vectorsym{Y}- \vectorsym{\hat{Y}}\| + D_{KL}[\vectorsym{N}(\vectorsym{Z}_\mu,\vectorsym{Z}_\sigma)||\vectorsym{N}(\vectorsym{0},\vectorsym{I})].
        \end{split}
    \end{equation}

    Therefore, we build a hierarchical, convolutional, multi-scale model for the encoder and decoder to model the long-range correlations in image data. We use four residual serial blocks to extract and reconstruct image features in different scales \citep{He2016Deep}. The first two terms in Equation \ref{J_x,y} measure reconstruction errors and force/torque prediction errors, respectively. The third term encourages the approximated posterior $q_\phi(\vectorsym{Z}|\vectorsym{X})$ to match the prior $p_\theta(\vectorsym{Z})$, which controls the capacity of latent information bottleneck. Although the derived optimization objective function Equation \ref{J_x,y} implicitly balances the three sources of loss, its optimization can be complex in practice. To resolve this issue, we propose the formulation of Equation \ref{J_alpha,x,y} in Section \ref{sec:Results-Learning} by introducing hyper-parameters $\alpha$ and $\beta$ to Equation \ref{J_x,y}.

    Introducing parameter $\beta \geq 0$ ahead of the third term of Equation \ref{J_alpha,x,y} is inspired by the work of Higgins et al. \citep{Higgins2016betaVAE} so that the optimal $\beta$ can be estimated heuristically in unsupervised scenarios. We tested several choices of $\beta$ in a candidate set, ranging from $10^{-4}$ to $10^2$, and fixed $\beta = 0.1$ in our experiment. 

    All networks were trained on a computer with NVIDIA GTX 1080Ti GPU, a batch size 64, and Adam optimizer \citep{Kingma2014Adam}. Considering the relatively small dataset size, the initial learning rate was set to $5 \times 10^{-5}$ and decreased with the training epoch.

\subsection{Tactile Grasping from On-Land to Underwater}

    We conducted object grasping experiments in a lab tank with and without contact force feedback for both on-land and underwater conditions to demonstrate the benefit of tactile learning in reliable object grasping against environmental uncertainties. We tested the grasping success rate using objects of different shapes, sizes, and materials from on-land to underwater. With the adoption of learned tactile perception, two more tasks were tested to demonstrate intelligent grasping behaviors in both on-land and underwater conditions. 

    \begin{algorithm}
        \caption{Reference Motion Generator}
        \label{alg:MotionGen}
        \begin{flushleft}
            \begin{tabular}{l@{}l}
            \textbf{Input:} & \hspace{2pt}  Raw image $I_{raw}$; Measured position $P_{m}$; 
            \\\mbox{} & \hspace{2pt} Reference force $F_{ref}$
            \\\textbf{Output:} & \hspace{2pt} Reference position $P_{ref}$
            \end{tabular} 
        \end{flushleft}
        \begin{algorithmic}[1]
            \State Initialize force tolerance $\delta$ and control gain $K$ 
            \While{$True$}
                \State Filtered image $I_{f} \gets ColorThreshold(I_{raw})$
                \State Estimated Force $F_{est} \gets VisualForceNet(I_{f})$
                \If{$|F_{ref} - F_{est}|> \delta$}
                    \State $\Delta_{P} = \frac{1}{K}(F_{ref} - F_{est})$ \Comment{position update rule}
                    \State $P_{ref} = P_{m} + \Delta_{P}$  
                \ElsIf{$|F_{ref} - F_{est}| \leq{\delta}$}
                \State $P_{ref} = P_{m}$
                \EndIf
            \EndWhile
        \end{algorithmic}
    \end{algorithm}

    As is shown in Figure \ref{fig:Control}(a), to achieve an intelligent closed-loop grasping behavior, it is an essential requirement for the grasping system to maintain a specified contact force while reacting to the varying environment. The industrial gripper can achieve reliable position commands at a high bandwidth due to the built-in low-level position controller. It is our goal to design a high-level position control policy $u=\pi(P_{m}, F_{est}, F_{ref})$ with measured gripper position $P_{m}$ and estimated contact force $F_{est}$ that achieves the desired contact force $F_{ref}$.
    \begin{equation}
        \pi(P_{m},F_{est},F_{ref
    }) = \mathop{\arg\min}\limits_{u}{|F_{est}-F_{ref}|}.
    \end{equation}

    Thanks to the proposed tactile force proprioceptive soft finger, which acts simultaneously as an end-effector and a sensor, a heuristic control policy $\pi=P_{ref}$ is presented to generate the reference motion command for the inner low-level position control loop, as shown in Algorithm \ref{alg:MotionGen}. The frequency of tactile perception feedback is determined by the computational time cost of the proposed SVAE model and the frame rate of the USB camera. We used a 1060Ti 6G GPU laptop in all grasping experiments, and the average inferring time was 5 ms. As a result, the force controller frequency is bounded by the camera frame rate at 120 Hz. Note that to estimate the contact force parallel to the gripping direction, modification of SVAE output is necessary. See Methods S3 and Methods S4 in the Supplementary Materials for a detailed derivation of controller design.

\section*{Acknowledgements}

    This work was partly supported by the Ministry of Science and Technology of China [2022YFB4701200], the National Natural Science Foundation of China [62206119], the Science, Technology, and Innovation Commission of Shenzhen Municipality [ZDSYS20220527171403009, JCYJ20220818100417038], Guangdong Provincial Key Laboratory of Human-Augmentation and Rehabilitation Robotics in Universities, and the SUSTech-MIT Joint Centers for Mechanical Engineering Research and Education.

\section*{Data Availability Statement}

    The data that support the findings of this study are openly available on GitHub at: 
    \url{https://github.com/bionicdl-sustech/AmphibiousSoftFinger}.

\section*{Supporting Information} 
\renewcommand{\thefigure}{S\arabic{figure}}
\renewcommand{\theequation}{S\arabic{equation}}
\setcounter{figure}{0}
\subsection*{Method S1 on Data Collection for Tactile Learning}
\label{sec:SM-Method1}

    We set up a collaborative robot arm, the Franka Emika Panda, for the automated data collection on land and the validation experiments underwater. As is shown in \textbf{Figure \ref{fig:SM-DataCollection}}(a), data collection using the robotic platform is guarded by sensors without human invention. 
    \begin{figure}[!ht]
        \centering
        \includegraphics[width=0.9\textwidth]{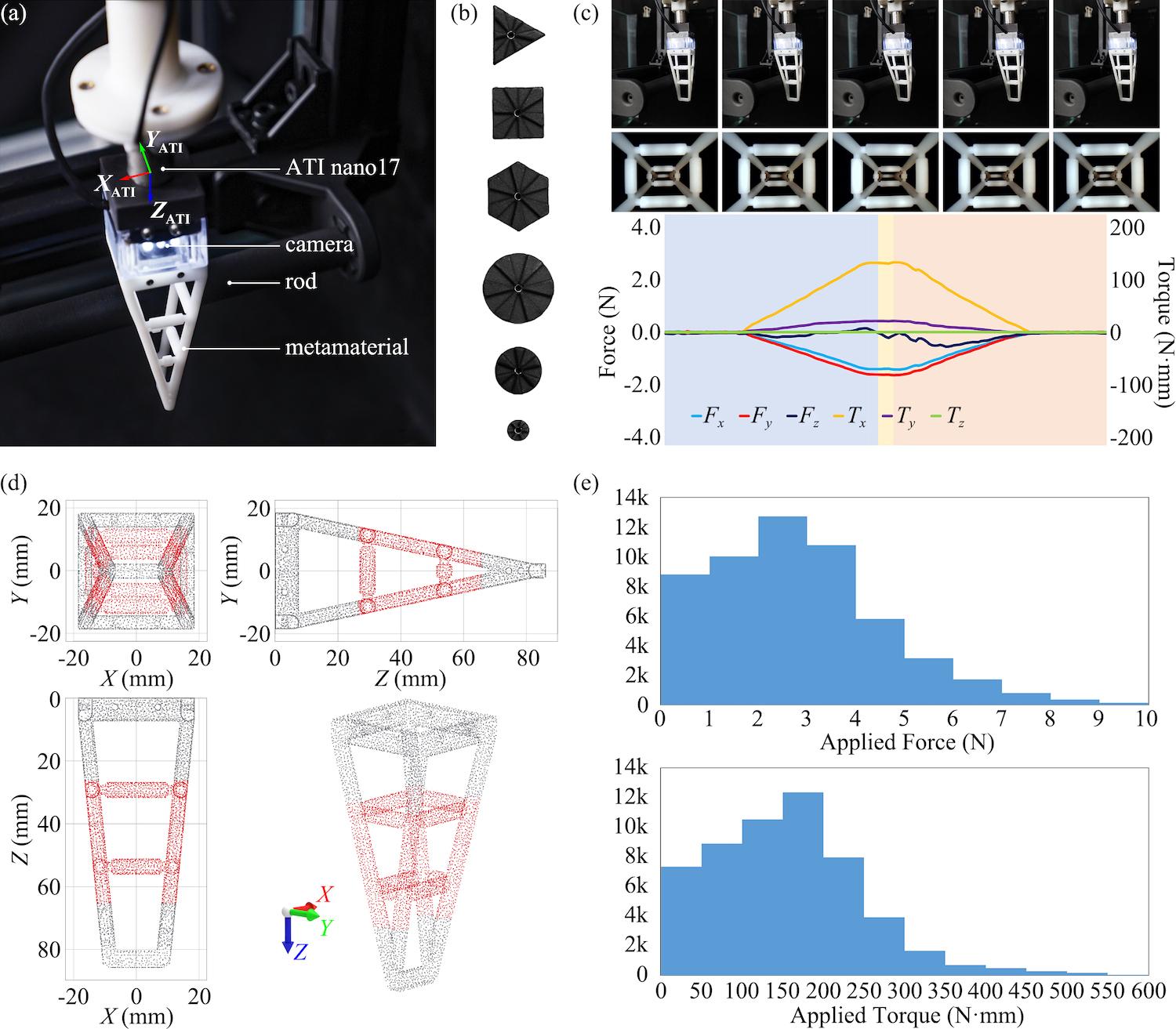}
        \caption{
            \textbf{Automated training data collection on land.} 
            (a) Setup of the data collection platform where a white extension link, an ATI Nano17 force/torque sensor, and the soft finger are mounted on a Franka Emika Panda robot arm in serial and then programmed to make contact with a fixed rod at a random pose. 
            (b) Cross-sectional geometries of the contacting rods. 
            (c) The procedure of collecting a single sample consists of approaching the target pose (blue shaded), recording data (yellow shaded), and leaving for the home pose (pink shaded). 
            (d) The contacting points cover the middle portion of the metamaterial highlighted in red color. 
            (e) Histogram of the contact forces and torques of the collected 30k samples.
            }
        \label{fig:SM-DataCollection}
    \end{figure}
    We use a 6-axis Force/Torque sensor (Nano17 from ATI) to provide labels for supervised visual-tactile learning, which has a force precision of 0.0125 N and a torque precision of 0.0625 N$\cdot$mm. 3D-printed rods with different cross-sectional geometries are used to contact soft fingers to increase the diversity of labeled contact samples, as shown in Figure \ref{fig:SM-DataCollection}(b). During the collection process, one of the rods was fixed each time while the soft finger was commanded to move and make contact with the rod at random poses. We simultaneously recorded the camera image and 6D force/torque measurement right after reaching the target pose, as highlighted in the yellow shaded area in Figure \ref{fig:SM-DataCollection}(c). The random contact pose of the soft finger concerning the rod has two degree-of-freedoms (DOFs) in position and one DOF in orientation and is generated by robot movement command $(x,0,z,0,0,\theta)$ where $x \sim U(0,5)$ cm, $y \sim U(-5,5)$ cm, and $\theta \sim U(-\pi,\pi)$, where $U$ stands for a uniform distribution. The total contact points cover the middle portion of the finger surface as illustrated in Figure \ref{fig:SM-DataCollection}(d). 30k pairs of samples in total are collected, which are split into training, validation, and test subsets at a ratio of 7:1:2. Figure \ref{fig:SM-DataCollection}(e) plots the histograms of the measured forces and torques, which have a range of 10 N for forces and 600 N$\cdot$mm for torques. 

\subsection*{Method S2 on Trade-off Analysis of the Latent Space Vector Dimension}
\label{sec:SM-Method2}

    A classic trade-off in machine learning scenarios is between model precision and model complexity \citep{Tishby2015Deep}. Before building the SVAE model, we trained several different sizes of latent dimension VAE models and tried to find the best deformation representation for tactile force prediction.  

    As illustrated in \textbf{Figure \ref{fig:SM-LatentSpaceAnalysis}}(a), if we try to encode the deformation images using smaller latent codes, a more considerable discrepancy in image reconstruction will be shown, in which case, tactile force prediction will most likely be inaccurate due to the loss of deformation-dependent information. On the other hand, if we approximate the original images with larger latent dimensions, the image reconstruction error is lower. Still, too many deformation-irrelevant details are kept, which leads to poor generalization. Figures \ref{fig:SM-LatentSpaceAnalysis}(b)-(g) show reconstructed images of the soft finger using a 6-dimension to 256-dimension latent space VAE model. Each row of images is generated by selecting a single latent axis and uniformly sampling values in the range of $[-5,5]$, which stand for the response of the reconstruction model with a specific size of latent space to its latent variable. With the increase of latent space size, the details of reconstructed images are more complete, but the response to a single latent dimension becomes less intuitive. In information theory, the complexity of the model is characterized by its coding length, which is proportional to the amount of information between the original data and their new representation \citep{Shwartz2022Opening}. A 32-dimensional latent code is finally chosen for tactile force prediction in our paper, as it is simple enough and provides relatively high accuracy. We admit optimal dimension of latent code that balances the trade-off can be sought more systematically.
    \begin{figure}[!ht]
        \centering
        \includegraphics[width=1\columnwidth]{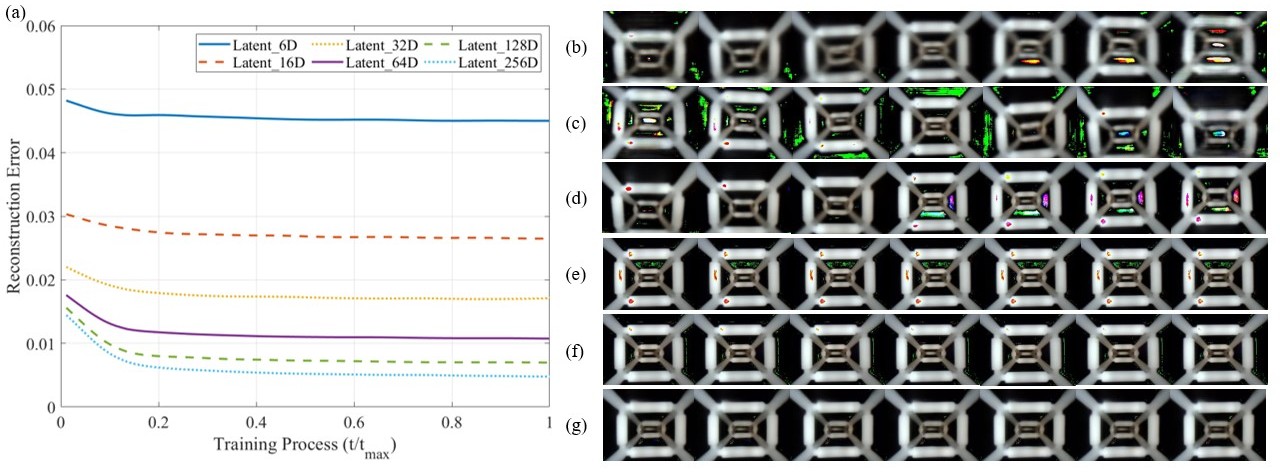}
        \caption{
            \textbf{Latent space dimension analysis using VAE models.} 
            (a) Reconstruction error (MSE) of VAE models with different sizes of latent space during the training. 
            Responses of (b) 6-dimension, (c) 16-dimension, (d) 32-dimension, (e) 64-dimension, (f) 128-dimension, and (g) 256-dimension latent space VAE models to their single latent variable ranging from $[-5,5]$.
            }
        \label{fig:SM-LatentSpaceAnalysis}
    \end{figure}
    
\subsection*{Method S3 on Contact Force Estimation}
\label{sec:SM-Method3}

    Figure \ref{fig:Control}(a) shows that image pre-processing for background segmentation is necessary to use the soft finger for force inference. Due to the distinct appearance difference between the finger and grasped object, a reliable image thresholding operation based on color is used. The output image from the color-based thresholding filter contains only the deformed metamaterial, which is sent to SVAE for force and torque prediction.

    Although the contact configuration between grasped object and the soft finger is arbitrarily complex, our perception finger can always predict a composite force and torque loaded at the base of the finger, which we denote as \!$\vectorsym{~^B\!F}\in{R^{6}}$, expressed in reference frame $B$ ($O_{B}\!\!-\!\!Y_{B}Z_{B}$ in Figure \ref{fig:Control}(a)). To measure the force parallel to the gripping direction $F_{est}\in{R}$, it is necessary to transform the spatial force \!\! $\vectorsym{~^B\!F}$ from the finger base reference frame $B$ to the world reference frame $W$ ($O_{W}\!\!-\!\!Y_{W}Z_{W}$ in Figure \ref{fig:Control}(a)) and project along the gripping direction.

    Denoting the rotation transformation matrix from frame $B$ to frame $W$ as\!\! $~^W\!R_{B}\in{SO(3)}$, the estimated contact force along the $Y$ axis of the world frame can be approximated as:
    \begin{equation}\label{Contact Force}
        F_{est}= \begin{bmatrix}
                    0 & 0 & 0\\
                    0 & 1 & 0\\
                    0 & 0 & 0
                    \end{bmatrix}   \begin{bmatrix}
                    ~^W\!R_{B} &     0_{3\times3}
                    \end{bmatrix} \vectorsym{~^B\!F}.
    \end{equation}

\subsection*{Method S4 on Details of Tactile Force Tracking Controller}
\label{sec:SM-Method4}

    Algorithm 1 describes the solution to the closed-loop tactile force tracking control policy $u=\pi(P_{m}, F_{est}, F_{ref})$ with measured gripper position $P_{m}$ and estimated contact force $F_{est}$ that achieves the desired contact force $F_{ref}$. Here we give design details of the proposed control policy $\pi=P_{ref}$. Figure \ref{fig:Control}(a) gives the approximate relationship between contact force and gripper position.
    \begin{equation}\label{position-force-func}
        F(P)=\left\{
                \begin{array}{ccc}
                    0 & & 0 \leq{P} < P_{c} \\
                    \Psi(P)& & P_{c} \leq{P} < P_{max}
                \end{array}
            \right.,
    \end{equation}
    where $P_{c}$ is the distance for the gripper to move from the non-contact state to the contact state. $P$ is the gripper position. $P_{max}$ is the maximum allowed position for the gripper to move. Inside the interval $[0, P_{c}]$, the object is not in contact with the gripper, while in the interval $[P_{c}, P_{max}]$, the contact force is monotonically increasing with position $P$, which indicates that:
    \begin{equation}\label{mono}
        \frac{\partial{F}}{\partial{P}}=\Lambda\geq{0},
    \end{equation}
    where $\Lambda$ is contact-state-dependent and related to the stiffness of the soft finger.

    \begin{prop}\label{proposition}
        Given control gain $K > \frac{\Lambda}{2}$,
        any feasible reference contact force $F_{ref}$ can be achieved within specified force tolerance $\delta > 0$, using proposed Algorithm 1 :
        \begin{equation}
            |F_{est}-F_{ref}| \leq{\delta}.
        \end{equation}
    \end{prop}
    
    \begin{proof}
        For every measured gripper position $P_{m}$, we can get an estimated contact force from SVAE:
        \begin{equation}
            F_{est} = F(P_{m}).
        \end{equation}
    
        The reference motion generator in Algorithm 1 for condition when $|F_{ref}-F_{est}| > \delta$, is expressed as following:
        \begin{equation}
            \begin{split}
                P_{ref} & = P_{m}+\Delta_{P}\\ & = P_{m}+\frac{F_{ref}-F(P_{m})}{K}.  
            \end{split}
        \end{equation}
    
        Due to the continuity of function $\Psi(P)$ in Equation \ref{position-force-func} within corresponding interval and monotonicity of Equation \ref{mono}, for $t\in{(0,1)}$, we can see that: 
        \begin{equation}\label{spreading}
            \begin{split}
                |F(P_{ref})-F_{ref}| & = |F(P_{m}+\Delta_{P})-F_{ref}|\\
                &= |F(P_{m})+\Delta_{P}\frac{\partial{F}}{\partial{P}}|_{P=P_{m}+t\Delta_{P}}-F_{ref}|\\
                &=|F(P_{m})-F_{ref}+\frac{\Lambda}{K}(F_{ref}-F(P_{m}))|\\
                &=|(1-\frac{\Lambda}{K})(F(P_{m})-F_{ref})|\\
                &\leq{|(1-\frac{\Lambda}{K})||F(P_{m})-F_{ref}|}.
            \end{split}
        \end{equation}
        
        Given the chosen control gain condition in Proposition \ref{proposition}, we have:
        \begin{equation}\label{range_K}
           K>\frac{\Lambda}{2}\Rightarrow{|(1-\frac{\Lambda}{K})|<1},
        \end{equation}
        then, the Equation \ref{spreading} becomes:
        \begin{equation}
          |F(P_{ref})-F_{ref}|<|F(P_{m})-F_{ref}|=|F_{est}-F_{ref}|.
        \end{equation}
        
        This indicates that the reference motion command generated by Algorithm 1 will always ensure a minor contact force error, eventually satisfying the specified force tolerance condition.  
    \end{proof}

\section*{Supporting Videos}
\label{sec:SM-Videos}

\begin{itemize}
    \item Movie S1. Amphibian Grasping with Visual-Tactile Soft Finger.
    \item Movie S2. Real-time Force/Torque Prediction.
    \item Movie S3. Object Grasping Success Rates Experiments with/without Contact Feedback.
    \item Movie S4. Contact Force Following Experiments.
    \item Movie S5. Object Shape Adaptation Experiments.
    \item Movie S6. Robot End-effector Reaction to Soft Finger Twist.
\end{itemize}

\bibliographystyle{unsrtnat}
\bibliography{references}

\end{document}